\documentclass[11pt]{article}

\usepackage[final]{acl}

\usepackage[T1]{fontenc}
\usepackage[utf8]{inputenc}
\usepackage{times}
\usepackage{latexsym}
\usepackage{microtype}
\usepackage{inconsolata}
\usepackage{xspace}

\usepackage{amsmath}
\usepackage{mathtools}
\usepackage{amssymb}
\usepackage{amsfonts}
\usepackage{amsthm}
\usepackage{siunitx}

\usepackage{booktabs}
\usepackage{tabularx}
\usepackage{array}
\usepackage{multirow}
\usepackage{arydshln}

\usepackage{subcaption}
\usepackage{graphicx}
\usepackage{graphicx}
\usepackage[table]{xcolor}
\usepackage{tcolorbox}

\usepackage{enumitem}
\usepackage{algorithm}
\usepackage{algorithmic}

\usepackage{natbib}
\usepackage{caption}
\usepackage{etoc}

\newtheorem{theorem}{Theorem}

\definecolor{mygray}{gray}{0.95}

\newcommand{\method}{BV-Blend\xspace}

\title{BV-Blend: Uncertainty-Weighted Historical Baselines for Stable Critic-Free RL with Verifiable Rewards}

\author{
    Yupeng Chang$^{1}$ \quad Yuan Wu$^{1}$\footnotemark[1] \quad Yi Chang$^{1,2,3}$ \\
    $^{1}$School of Artificial Intelligence, Jilin University \\
    $^{2}$Engineering Research Center of Knowledge-Driven Human-Machine Intelligence, MOE, China \\
    $^{3}$International Center of Future Science, Jilin University \\
    changyp23@mails.jlu.edu.cn, \{yuanwu, yichang\}@jlu.edu.cn
}

\begin{document}
\etocdepthtag.toc{chapter}
\etocsettagdepth{chapter}{none}
\etocsettagdepth{appendix}{none}
\maketitle

\renewcommand{\thefootnote}{\fnsymbol{footnote}}
\footnotetext[1]{Corresponding author}

\begin{abstract}
Critic-free reinforcement learning with verifiable rewards (RLVR), exemplified by Group Relative Policy Optimization (GRPO), avoids training a value function (critic) and reduces memory and compute overhead relative to critic-based PPO pipelines for aligning large language models.
However, GRPO-style advantage estimation depends on prompt-local (within-prompt-group) reward statistics and can be unstable.
In particular, when all rollouts in a prompt group receive identical rewards, the within-group reward variance becomes zero, and group normalization yields \emph{zero} advantages for that group, impeding learning in cold-start regimes with binary verifiers.
We introduce \textbf{BV-Blend}, a critic-free framework that stabilizes advantage estimation by combining prompt-local on-policy statistics with semantic-cluster-conditioned historical moments.
BV-Blend maintains EMA-tracked reward moments for each cluster, derives a confidence weight from a standard error of the mean (SEM) proxy, and uses this weight to blend historical and prompt-local baseline and variance statistics into a standardized advantage for PPO-style clipped updates.
Experiments on verifiable reasoning benchmarks show that BV-Blend improves training stability and performance, and remains robust in regimes where group-normalized methods may stall.
\end{abstract}

\section{Introduction}

Reinforcement Learning with Verifiable Rewards (RLVR) has recently become a practical paradigm for aligning Large Language Models (LLMs) in domains with \emph{objective} correctness signals, such as mathematics and code, where outputs can be automatically verified (e.g., exact match against a reference answer, formal proof checking, or passing unit tests)~\citep{yu2025formalmath,yan2025learning,gui2024logicgame}.
Compared with process-level supervision such as Process Reward Models (PRMs)~\citep{lightman2023let}, RLVR optimizes outcome-based rewards and does not require scoring intermediate reasoning steps.
When verifiers are reliable, collecting multiple rollouts per prompt naturally improves exploration and increases the chance of observing correct solutions, which has underpinned several recent reasoning-focused RL systems~\citep{guo2025deepseek}.
Beyond structured domains, RLVR has also been extended to more diverse settings when reference signals or robust verifiers are available~\citep{su2025crossing}.

However, directly applying the de facto RLHF recipe---PPO with a learned critic~\citep{ziegler2019fine,ouyang2022training}---to RLVR exposes a practical tension.
With sparse, trajectory-level verifiable rewards, learning an accurate value function can be challenging, while the additional critic introduces non-trivial memory and compute overhead, as well as additional training complexity, at LLM scale~\citep{ouyang2022training}.
These considerations have accelerated interest in critic-free alignment.
Preference-optimization methods such as DPO~\citep{rafailov2024direct} avoid explicit reward modeling and online RL, but are primarily studied under preference or utility supervision rather than settings where informative, programmatic outcome rewards are directly available.
We therefore focus on \emph{direct-reward, critic-free} policy optimization, where GRPO~\citep{shao2024deepseekmath} is a prominent baseline that derives advantages through prompt-local reward normalization within each prompt group.

Despite its simplicity, prompt-local normalization can be unstable.
The resulting advantage estimator may depend strongly on transient within-group statistics, leading to high variance and unstable learning dynamics, which has motivated a growing body of analyses and fixes~\citep{liang2025group,chen2025spectral,mroueh2025reinforcement,yu2025dapo}.
Moreover, RLVR commonly encounters a cold-start regime in which the current policy produces predominantly incorrect solutions; under binary verifiers, many prompt groups become effectively deterministic (e.g., all failures), yielding (near-)zero within-group reward variance.
In this \emph{zero-variance} regime, standard within-group normalization produces (near-)zero advantages, removing the learning signal for those prompt groups~\citep{le2025no}.

To address this instability, we propose \textbf{BV-Blend}, a critic-free framework that stabilizes advantage estimation by combining (i) an on-policy prompt-local signal with (ii) low-variance historical moments aggregated over semantically similar prompts.
BV-Blend computes an uncertainty-aware confidence weight from a standard-error-of-the-mean (SEM) proxy of the historical statistics, downweighting unreliable historical information while relying more on it when the historical estimate is better supported.
As a result, BV-Blend mitigates advantage collapse on zero-variance prompt groups without training a critic and remains compatible with PPO-style optimization.

Our contributions are:
\begin{itemize}
    \item We analyze the \emph{zero-variance} failure mode of prompt-local (group-normalized) advantage estimation in direct-reward, critic-free RL, where the normalized learning signal vanishes when all rollouts in a prompt group receive identical rewards.
    \item We introduce \textbf{BV-Blend}, an uncertainty-aware historical blending mechanism that stabilizes prompt-local advantage estimation by leveraging semantic-cluster-conditioned historical moments, without training a critic.
    \item We empirically demonstrate that BV-Blend improves training stability and performance on verifiable reasoning benchmarks, and remains robust in regimes where standard group-normalized methods may stall.
\end{itemize}

\section{Related Work}

\paragraph{Critic-based and preference-based alignment.}
Policy-gradient methods typically reduce variance with a learned value-function baseline; GAE~\citep{schulman2015high} (commonly paired with PPO~\citep{schulman2017ppo}) yields lower-variance advantages than Monte Carlo estimators such as REINFORCE~\citep{williams1992simple}. In RLHF-style LLM alignment, PPO pipelines often train an auxiliary value head alongside the policy during RL fine-tuning~\citep{ouyang2022training}. A separate line of work avoids online RL and critic training by directly optimizing objectives from preference or utility supervision, e.g., DPO~\citep{rafailov2024direct}, KTO~\citep{ethayarajh2024kto}, and SimPO~\citep{meng2024simpo}. These methods are primarily studied under preference or utility supervision (or implicit rewards), rather than explicit outcome-based verifiable rewards.

\paragraph{Direct-reward, critic-free policy optimization.}
Our work builds on critic-free policy optimization with verifiable rewards, exemplified by Group Relative Policy Optimization (GRPO)~\citep{shao2024deepseekmath}. GRPO avoids a learned critic by normalizing rewards within each prompt group, but the learning signal can be sensitive to transient within-group statistics and may collapse when within-group reward variance is near zero. Recent work has explored related issues and remedies, including correcting optimization bias in GRPO-style objectives (Dr.~GRPO)~\citep{liu2025understanding}, temporally smoothing baselines via lightweight Bayesian or Kalman-style updates (KRPO)~\citep{wang2025kalman}, extending GRPO-style optimization to multi-turn tool-use settings (ARPO)~\citep{dong2025agentic}, and extracting learning signals from zero-variance prompts (RL-ZVP)~\citep{le2025no}. Other work also mitigates zero-variance or low-signal training regimes through data- or sampling-level interventions, such as dynamic sampling~\citep{yu2025dapo}.

Our method is most closely related to approaches that stabilize GRPO-style advantage estimation without training a critic. Compared with KRPO, which primarily smooths statistics across training steps, BV-Blend maintains historical reward moments conditioned on semantic clusters of prompts. Compared with RL-ZVP, which specifically targets zero-variance prompts, BV-Blend uses uncertainty-aware blending of prompt-local and cluster-conditioned historical statistics to form a unified standardized advantage for all prompt groups. Compared with dynamic-sampling-based remedies, BV-Blend operates at the level of the advantage estimator rather than the data-selection policy. Overall, BV-Blend is a critic-free method that modifies the advantage estimator while remaining compatible with standard PPO-style training pipelines.

\begin{figure*}[t]
    \centering
    \includegraphics[width=0.9\linewidth]{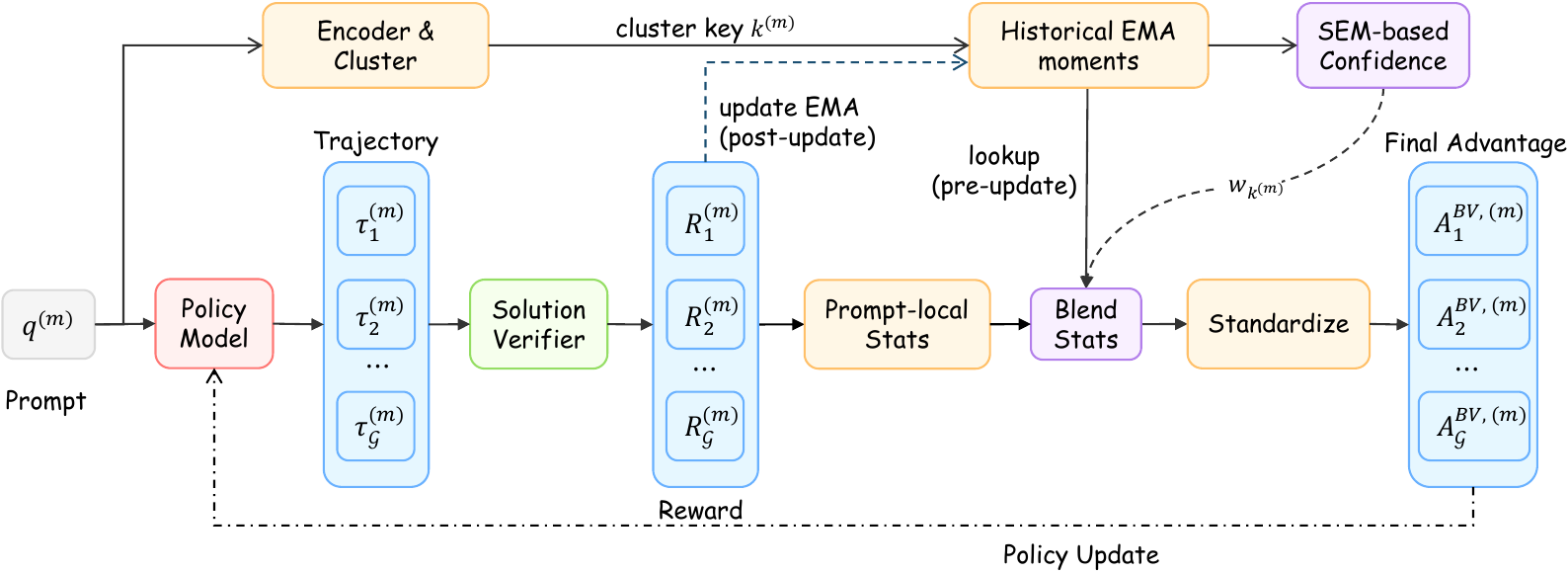}
\caption{%
\textbf{BV-Blend overview.}
For each prompt \(q^{(m)}\), we sample \(G\) trajectories \(\{\tau_i^{(m)}\}\) with the behavior policy and obtain verifier rewards \(\{R_i^{(m)}\}\).
We compute prompt-local statistics \((\mu_{\mathcal{G}}^{(m)}, \sigma_{\mathcal{G}}^{(m)})\), embed \(q^{(m)}\), and assign it to a semantic cluster \(k^{(m)}\).
Using \emph{pre-update} EMA moments \((\mu_{\mathrm{hist}}(k), v_{\mathrm{hist}}(k), N_k^{\mathrm{eff}})\), we compute the SEM-based confidence \(w_k\) (Eq.~\eqref{eq:weight}; cold start: \(w_k{=}0\) for unseen clusters), blend baseline and variance statistics to obtain \((b^{(m)}, s^{(m)})\) (Eq.~\eqref{eq:blend_stats}), and form advantages \(A_i^{\mathrm{BV},(m)}\) (Eq.~\eqref{eq:bv_adv}) for a PPO-style update.
EMA moments are updated \emph{post-update} using the current batch.%
}
    \label{fig:framework}
\end{figure*}

\section{Method}
\label{sec:method}

Critic-free policy optimization for LLM alignment often relies on prompt-local (group-dependent) advantage normalization (e.g., GRPO-style estimators), but when within-group reward dispersion is small, prompt-local standardization can yield near-zero advantages and effectively remove the learning signal for that prompt group. We propose \textbf{BV-Blend}, which constructs a \emph{single} per-trajectory advantage by combining prompt-local statistics with semantic-cluster-conditioned historical moments, using a confidence weight derived from a standard-error-of-the-mean (SEM) proxy. Concretely, BV-Blend blends baseline and variance statistics, then standardizes once using the resulting scale. BV-Blend keeps the PPO-style clipped objective unchanged~\citep{schulman2017ppo} and modifies only the advantage estimator: the SEM weight is computed from \emph{pre-update} EMA statistics, while EMA moments are updated \emph{post-update} using the current batch (Fig.~\ref{fig:framework}). As in other group-based normalization methods, we stop gradients through all reward statistics and do not assume the resulting normalized estimator is unbiased.

\subsection{Background: prompt-local normalization in critic-free RL}
\label{sec:background}

A training batch contains \(M\) prompts \(\{q^{(m)}\}_{m=1}^{M}\).
For each prompt \(q^{(m)}\), the behavior policy \(\pi_{\theta_{\mathrm{old}}}\) samples \(G\) trajectories
\(\mathcal{G}(q^{(m)})=\{\tau^{(m)}_i\}_{i=1}^{G}\).
Each trajectory \(\tau^{(m)}_i=(y^{(m)}_{i,1},\dots,y^{(m)}_{i,T^{(m)}_i})\) receives a scalar trajectory-level reward \(R^{(m)}_i\) from an external verifier.

For token position \(t\) in \(\tau^{(m)}_i\), define the importance ratio
\begin{equation}
r^{(m)}_{i,t}(\theta)=\frac{\pi_\theta\!\left(y^{(m)}_{i,t}\mid q^{(m)},y^{(m)}_{i,<t}\right)}{\pi_{\theta_{\mathrm{old}}}\!\left(y^{(m)}_{i,t}\mid q^{(m)},y^{(m)}_{i,<t}\right)},
\label{eq:ratio}
\end{equation}
and its clipped version
\(\tilde r^{(m)}_{i,t}(\theta)=\operatorname{clip}\!\bigl(r^{(m)}_{i,t}(\theta),1-\epsilon,1+\epsilon\bigr)\).
Since rewards are trajectory-level, we compute a single advantage \(A^{(m)}_i\) per trajectory and apply it to all completion tokens.
We restrict optimization to completion tokens using a mask \(m^{(m)}_{i,t}\in\{0,1\}\) and define masked token means as
\begin{equation}
\mathbb{E}_{t\sim\tau^{(m)}_i}[f_{i,t}]
\triangleq
\frac{\sum_{t=1}^{T_i^{(m)}} m^{(m)}_{i,t}\, f_{i,t}}{\sum_{t=1}^{T_i^{(m)}} m^{(m)}_{i,t}},
\label{eq:masked_mean}
\end{equation}
where each completion contains at least one generated token so the denominator is non-zero.
Throughout, \(\mathbb{E}_{t\sim\tau^{(m)}_i}[\cdot]\) denotes a \emph{mean} over completion tokens (Eq.~\eqref{eq:masked_mean}) to avoid length-dependent gradient scaling; concretely, \(m^{(m)}_{i,t}=1\) for generated completion tokens up to (and including) EOS, and \(0\) for prompt tokens and padding.
In implementation, \(A^{(m)}_i\) is treated as a trajectory-level constant (no gradient through reward statistics).

GRPO-style training forms a prompt-local standardized advantage~\citep{shao2024deepseekmath}:
\begin{equation}
A^{\text{GRPO},(m)}_i = \frac{R^{(m)}_i - \mu_{\mathcal{G}}^{(m)}}{\sigma_{\mathcal{G}}^{(m)} + \delta},
\label{eq:grpo_adv}
\end{equation}
where \(\mu_{\mathcal{G}}^{(m)}\) and \(\sigma_{\mathcal{G}}^{(m)}\) are the mean and standard deviation of \(\{R^{(m)}_i\}_{i=1}^{G}\), and \(\delta>0\) is a small constant.
If \(G<2\), we set \(\sigma_{\mathcal{G}}^{(m)}=0\), yielding \(A^{\text{GRPO},(m)}_1=0\).
The clipped surrogate objective is
\begin{equation}
\label{eq:grpo_obj}
\resizebox{\columnwidth}{!}{$\displaystyle
\begin{aligned}
\mathcal{J}_{\text{GRPO}}(\theta)
&=
\mathbb{E}_{m}\!\Biggl[
\frac{1}{G}\sum_{i=1}^{G}
\mathbb{E}_{t\sim\tau^{(m)}_i}\!\Bigl[
\min\!\Bigl(
r^{(m)}_{i,t}(\theta)\,A^{\text{GRPO},(m)}_i,\\
&\qquad\qquad\qquad\qquad
\tilde r^{(m)}_{i,t}(\theta)\,A^{\text{GRPO},(m)}_i
\Bigr)
\Bigr]
\Biggr].
\end{aligned}
$}
\end{equation}

When \(\sigma_{\mathcal{G}}^{(m)}\approx 0\), Eq.~\eqref{eq:grpo_adv} collapses toward zero advantages, effectively removing the learning signal for that prompt group.

\subsection{BV-Blend: SEM-driven blending of historical and prompt-local baselines}
\label{sec:bvblend}

BV-Blend computes a single standardized advantage \(A^{\mathrm{BV},(m)}_i\) by blending \emph{baseline and variance statistics} and then standardizing once (taking a square root to obtain the scale).
It follows the common form \(A=(R-b)/s\): GRPO uses prompt-local \((b,s)=(\mu_{\mathcal{G}}^{(m)},\sigma_{\mathcal{G}}^{(m)})\), whereas BV-Blend interpolates between prompt-local statistics and \emph{semantic-cluster-conditioned} EMA moments, with the interpolation controlled by a SEM-based confidence \(w_k\).

\subsubsection{Prompt clustering}
\label{sec:clustering}

Each prompt \(q^{(m)}\) is embedded by a frozen encoder \(E(\cdot)\) and assigned to a fixed \(K\)-means codebook \(\{c_j\}_{j=1}^{K}\):
\begin{equation}
k^{(m)}=\arg\min_{j\in\{1,\dots,K\}}\ \|E(q^{(m)})-c_j\|_2^2.
\label{eq:cluster}
\end{equation}
The codebook is trained offline on a representative prompt corpus and kept fixed during RL to avoid cluster-identity drift.
We report the encoder choice, \(K\), codebook training corpus, and \(K\)-means settings (implementation, seed, and iterations) for reproducibility.

\subsubsection{Historical statistics with EMA moments}
\label{sec:ema}

For each cluster \(k\), we maintain EMA moments: mean \(m_{1}(k)\), second raw moment \(m_{2}(k)\), and EMA mass \(N^{\mathrm{eff}}_k\) (here \(m_{1},m_{2}\) denote moments, not the prompt index \(m\)).
Given batch \(\mathcal{B}\), let \(\mathcal{I}_{\mathcal{B}}(k)=\{(m,i): k^{(m)}=k\}\) and \(N_{\mathcal{B}}(k)=|\mathcal{I}_{\mathcal{B}}(k)|\).
Define sufficient statistics
\[
\begin{aligned}
S_{1,\mathcal{B}}(k)
&= \sum_{(m,i)\in \mathcal{I}_{\mathcal{B}}(k)} R^{(m)}_i,\\
S_{2,\mathcal{B}}(k)
&= \sum_{(m,i)\in \mathcal{I}_{\mathcal{B}}(k)} \bigl(R^{(m)}_i\bigr)^2.
\end{aligned}
\]

In distributed training, we aggregate \(\{S_{1,\mathcal{B}}(k),S_{2,\mathcal{B}}(k),N_{\mathcal{B}}(k)\}\) across workers and apply the EMA update \emph{after} the policy optimization step.
If \(N_{\mathcal{B}}(k)>0\), define the batch mean and second raw moment
\[
\mu_{\mathcal{B}}(k)=\frac{S_{1,\mathcal{B}}(k)}{N_{\mathcal{B}}(k)},\qquad
\mu_{2,\mathcal{B}}(k)=\frac{S_{2,\mathcal{B}}(k)}{N_{\mathcal{B}}(k)}.
\]
With EMA rate \(\gamma\in(0,1]\),
\begin{align}
m_{1}(k) &\leftarrow (1-\gamma)m_{1}(k)+\gamma\,\mu_{\mathcal{B}}(k), \nonumber\\
m_{2}(k) &\leftarrow (1-\gamma)m_{2}(k)+\gamma\,\mu_{2,\mathcal{B}}(k), \label{eq:ema_moments}\\
N^{\mathrm{eff}}_k &\leftarrow (1-\gamma)N^{\mathrm{eff}}_k+\gamma\,N_{\mathcal{B}}(k). \nonumber
\end{align}
We define
\begin{equation}
\label{eq:hist_var}
\begin{aligned}
v_{\mathrm{hist}}(k)
&= \max\!\bigl(m_{2}(k)-m_{1}(k)^2,\,0\bigr),\\
\sigma_{\mathrm{hist}}(k)
&= \sqrt{v_{\mathrm{hist}}(k)},\\
\mu_{\mathrm{hist}}(k)
&= m_{1}(k).
\end{aligned}
\end{equation}

\paragraph{Initialization and cold start.}
On first observation of cluster \(k\), we initialize
\(
N^{\mathrm{eff}}_k\!\leftarrow\! N_{0},
\;
m_{1}(k)\!\leftarrow\! \mu_{\mathcal{B}}(k),
\;
m_{2}(k)\!\leftarrow\! \mu_{\mathcal{B}}(k)^2 + V_{\mathrm{prior}},
\)
with \(N_{0}>0\) and \(V_{\mathrm{prior}}>0\).

\paragraph{Update ordering.}
For each batch, we compute \(w_k\) using EMA statistics \emph{before} incorporating the current batch.
If a cluster \(k\) is first observed in the current batch (i.e., no prior EMA state exists at advantage-computation time), we set \(w_k=0\) for this batch (pure prompt-local normalization for \(A^{\mathrm{BV}}\)) and create its EMA state using the initialization above \emph{after} the policy optimization step.
For previously seen clusters, EMA moments are updated \emph{after} the policy optimization step using the aggregated sufficient statistics.

\subsubsection{Uncertainty-to-confidence mapping}
\label{sec:uncertainty}

We quantify historical uncertainty using a SEM-style proxy
\begin{equation}
\mathrm{SEM}_{\mathrm{hist}}(k)=
\frac{\sigma_{\mathrm{hist}}(k)}{\sqrt{N^{\mathrm{eff}}_k+\delta_N}},
\label{eq:sem}
\end{equation}
with \(\delta_N>0\).
We map uncertainty to a confidence weight
\begin{equation}
w_k=\exp\!\left(-\frac{\mathrm{SEM}_{\mathrm{hist}}(k)}{T}\right),
\label{eq:weight}
\end{equation}
where \(T>0\) controls sensitivity to the reward scale.
For clusters with an EMA state, \(w_k\in(0,1]\): lower uncertainty yields larger \(w_k\) (more reliance on historical moments), while higher uncertainty yields smaller \(w_k\) (more reliance on prompt-local statistics).
We report \(T,\gamma,N_0,V_{\mathrm{prior}},\delta_N\) in experiments.

\subsubsection{Baseline-and-scale blending and BV advantage}
\label{sec:blend}

For each prompt group \(m\), compute prompt-local statistics \(\mu_{\mathcal{G}}^{(m)}\) and \(\sigma_{\mathcal{G}}^{(m)}\) over \(\{R^{(m)}_i\}_{i=1}^{G}\).
Let \(w^{(m)}=w_{k^{(m)}}\).
We blend baseline and variance statistics, and take a square root to obtain the scale:
\begin{equation}
\begin{aligned}
    b^{(m)} &= w^{(m)}\,\mu_{\mathrm{hist}}(k^{(m)}) + \bigl(1-w^{(m)}\bigr)\,\mu_{\mathcal{G}}^{(m)}, \\
    s^{(m)} &= \sqrt{\,w^{(m)}\,v_{\mathrm{hist}}(k^{(m)}) + \bigl(1-w^{(m)}\bigr)\bigl(\sigma_{\mathcal{G}}^{(m)}\bigr)^2\,},
\end{aligned}
\label{eq:blend_stats}
\end{equation}
and define the BV-Blend advantage
\begin{equation}
A^{\mathrm{BV},(m)}_i=\frac{R^{(m)}_i-b^{(m)}}{s^{(m)}+\delta}.
\label{eq:bv_adv}
\end{equation}
When \(\sigma_{\mathcal{G}}^{(m)}\approx 0\), the blended scale \(s^{(m)}\) remains non-degenerate whenever the historical term provides non-zero variance, thereby preventing collapse of the learning signal.

\section{Experiments}
\label{sec:experiments}

\paragraph{Domain and Datasets.}
We focus on mathematical reasoning, where evaluation is rigorous and objective.
Compared with more subjective tasks, math offers (i) deterministic, programmatic verification against ground-truth final answers, reducing reliance on preference labels and learned judges that may be biased or exploitable~\citep{DBLP:conf/iclr/HuangQWPT25, DBLP:conf/iclr/ZhengPDLJL25}; (ii) scalable automated evaluation without costly human annotation; and (iii) diverse multi-step problems with unambiguous correctness criteria.
Our training set contains 45{,}000 problems curated from the default 94k split of OpenR1-Math-220k~\citep{openr1, yan2025learning}.
OpenR1-Math-220k is built from NuminaMath 1.5 prompts and includes 2--4 reasoning traces generated by DeepSeek-R1; most traces are verified by \textit{Math-Verify} (with a small portion additionally judged by an LLM), and each problem has at least one correct trace~\citep{openr1, numina_math_datasets, guo2025deepseek}.
We use \textit{Math-Verify} to remove instances with invalid or unverifiable final answers, and further filter samples whose traces exceed 8{,}192 tokens.
This curated set is used as (1) prompts for on-policy rollouts, (2) ground-truth answers for deterministic reward computation, and (3) a corpus of reasoning traces for training SFT baselines.

\begin{table*}[t]
\centering
\caption{Main results on Qwen2.5-Math-7B. We compare \method with baselines on mathematical reasoning and OOD generalization benchmarks. Best and second-best in each column are in \textbf{bold} and \underline{underlined}. Math Reasoning Avg.\ averages AIME 2024/2025, AMC, MATH-500, Minerva, and Olympiad; Generalization Avg.\ averages ARC-C, GPQA*, and MMLU-Pro.}

\label{tab:main_results}
\resizebox{.92\textwidth}{!}{
\setlength{\tabcolsep}{2.5pt}
\renewcommand{\arraystretch}{1.3}
\begin{tabular}{@{}lcccccc cccc@{}}
\toprule
\multirow{2}{*}{\textbf{Model}} & \multicolumn{6}{c}{\textbf{Math Reasoning Performance}} & \multicolumn{4}{c}{\textbf{Generalization Performance}} \\
\cmidrule(lr){2-7} \cmidrule(lr){8-11}
 & \textbf{AIME 24/25} & \textbf{AMC} & \textbf{MATH-500} & \textbf{Minerva} & \textbf{Olympiad} & \textbf{Avg.} & \textbf{ARC-C} & \textbf{GPQA*} & \textbf{MMLU-Pro} & \textbf{Avg.} \\
\midrule
Qwen-Base & 11.5/4.9 & 31.3 & 43.6 & 7.4 & 15.6 & 19.1 & 18.2 & 11.1 & 16.9 & 15.4 \\
Qwen-Instruct & 12.5/10.2 & 48.5 & 80.4 & 32.7 & 41.0 & 37.6 & 70.3 & 24.7 & 34.1 & 43.0 \\
\midrule
\multicolumn{11}{c}{\textit{Baselines: On-Policy RLVR}} \\
\midrule
GRPO (our replication) & 25.1/15.3 & 62.0 & 84.4 & 39.3 & 46.8 & 45.5 & \underline{82.3} & 40.4 & 49.3 & 57.3 \\
SimpleRL-Zero & 27.0/6.8 & 54.9 & 76.0 & 25.0 & 34.7 & 37.4 & 30.2 & 23.2 & 34.5 & 29.3 \\
OpenReasoner-Zero & 16.5/15.0 & 52.1 & 82.4 & 33.1 & 47.1 & 41.0 & 66.2 & 29.8 & \textbf{58.7} & 51.6 \\
PRIME-Zero & 17.0/12.8 & 54.0 & 81.4 & 39.0 & 40.3 & 40.8 & 73.3 & 18.2 & 32.7 & 41.4 \\
Oat-Zero & \underline{33.4}/11.9 & 61.2 & 78.0 & 34.6 & 43.4 & 43.8 & 70.1 & 23.7 & 41.7 & 45.2 \\
\midrule
\multicolumn{11}{c}{\textit{Baselines: Hybrid \& Off-Policy Methods}} \\
\midrule
SFT & 22.2/22.3 & 52.8 & 82.6 & \underline{40.8} & 43.7 & 44.1 & 75.2 & 24.7 & 42.7 & 47.5 \\
SFT+RL & 25.8/23.1 & 62.7 & 87.2 & 39.7 & 50.4 & 48.2 & 72.4 & 24.2 & 37.7 & 44.8 \\
ReLIFT & 28.2/20.1 & 64.9 & 87.4 & 33.8 & 52.5 & 47.8 & 76.2 & 37.9 & 52.5 & 55.5 \\
LUFFY & 29.4/23.1 & 65.6 & \underline{87.6} & 37.5 & \underline{57.2} & 50.1 & 80.5 & 39.9 & 53.0 & 57.8 \\
LUFFY$^{\dagger}$ & 30.7/\textbf{25.5} & \underline{66.2} & 86.8 & \textbf{41.2} & 55.3 & \underline{51.0} & 81.8 & \underline{49.0} & 54.7 & \underline{61.8} \\
\midrule
\rowcolor{green!10} \method (\textit{ours}) & \textbf{34.2}/\underline{23.6} & \textbf{66.5} & \textbf{87.9} & 40.7 & \textbf{57.4} & \textbf{51.7} & \textbf{83.1} & \textbf{52.6} & \underline{56.5} & \textbf{64.1} \\
\bottomrule
\end{tabular}
}
\end{table*}

\paragraph{Evaluation.}
We evaluate \method on benchmarks probing two complementary aspects:
(i) \emph{in-domain} mathematical reasoning across a broad difficulty range---from pre-college competition problems (AMC~\citep{DBLP:conf/acl/HeLBHTSHHHZLQL024}), through university-level problem solving (MATH-500~\citep{DBLP:conf/nips/HendrycksBKABTS21}, Minerva~\citep{DBLP:conf/nips/LewkowyczADDMRS22}), to elite competition-level challenges (AIME 2024/2025~\citep{li2024numinamath});
and (ii) \emph{out-of-distribution} (OOD) generalization to non-mathematical benchmarks, including abstract reasoning (ARC-C~\citep{DBLP:journals/corr/abs-1803-05457}), expert-level scientific knowledge (GPQA-diamond~\citep{DBLP:journals/corr/abs-2311-12022}), and multidisciplinary problem solving (MMLU-Pro~\citep{DBLP:conf/nips/WangMZNCGRAHJLK24}).
Our main results use Qwen2.5-Math-7B~\citep{yang2024qwen2} as the backbone; additional backbones are reported in the \emph{Implementation Details} paragraph below.
Unless otherwise specified, we decode for evaluation with temperature \(0.6\).
For multiple-choice benchmarks, we randomly permute answer options for each question (with labels remapped accordingly) to mitigate position bias.
We report pass@1 on large-scale benchmarks (MATH-500, Minerva, and the OOD suite).
For smaller and more challenging test sets (AIME and AMC), we instead report avg@32, defined as the mean success rate over 32 independent samples per problem, which is more stable under stochastic decoding.

\paragraph{Baselines.}
To contextualize \method, we compare against baselines built on the same Qwen2.5-Math-7B backbone (Table~\ref{tab:main_results}), grouped into two categories.
\textbf{(1) On-policy RLVR methods} start from the base model and optimize using only on-policy rollouts with verifiable rewards (i.e., without SFT initialization or off-policy demonstrations).
This category includes our replication of GRPO~\citep{shao2024deepseekmath} and representative ``Zero'' RLVR systems trained under different reward designs and training recipes:
SimpleRL-Zero~\citep{zeng2025simplerl}, Open-Reasoner-Zero~\citep{hu2025open}, PRIME-Zero~\citep{cui2025process}, and Oat-Zero~\citep{drgrpo}.
\textbf{(2) Hybrid \& off-policy methods} leverage additional external data beyond on-policy rollouts, including SFT, a sequential SFT$\rightarrow$RL pipeline, ReLIFT~\citep{ma2025learning}, and LUFFY~\citep{yan2025learning}.
For LUFFY, we report both the standard and extended-training ($\dagger$) variants when available.

\begin{figure*}[t]
    \centering
    \begin{subfigure}[b]{0.24\textwidth}
        \centering
        \includegraphics[width=\linewidth]{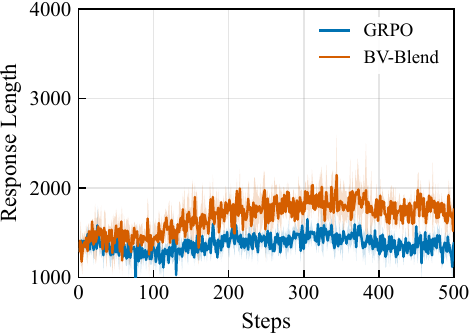}
        \caption{Response length}
        \label{fig:resp_len}
    \end{subfigure}\hfill
    \begin{subfigure}[b]{0.24\textwidth}
        \centering
        \includegraphics[width=\linewidth]{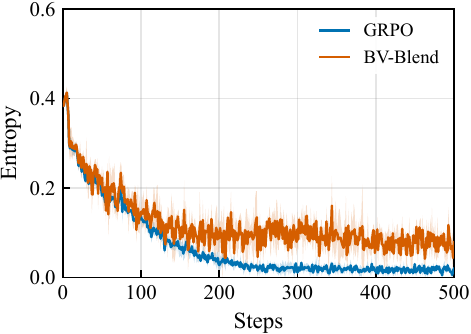}
        \caption{Training entropy}
        \label{fig:train_ent}
    \end{subfigure}\hfill
    \begin{subfigure}[b]{0.24\textwidth}
        \centering
        \includegraphics[width=\linewidth]{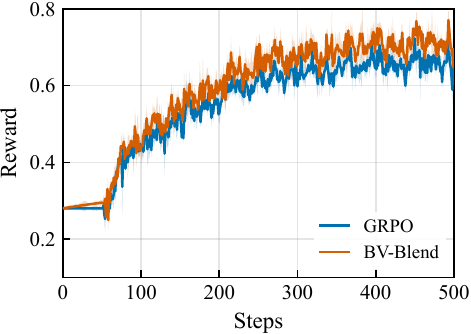}
        \caption{Training reward}
        \label{fig:train_rew}
    \end{subfigure}\hfill
    \begin{subfigure}[b]{0.236\textwidth}
        \centering
        \includegraphics[width=\linewidth]{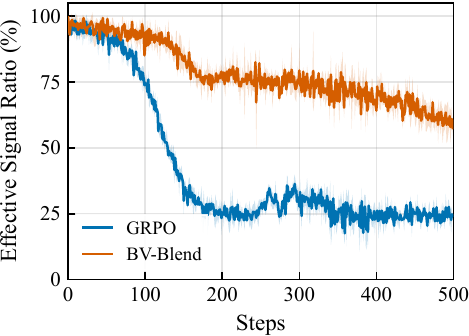}
        \caption{Effective signal ratio}
        \label{fig:eff_signal}
    \end{subfigure}

\caption{\textbf{GRPO vs.\ BV-Blend.}
    We track (a) response length (tokens), (b) policy training entropy, (c) mean training reward (verifier score),
    and (d) the effective-signal ratio: the fraction of prompts whose method-specific normalization scale remains non-degenerate during training (GRPO: \(\sigma_{\mathcal{G}}^{(m)}\); BV-Blend: \(s^{(m)}\) in Eq.~\eqref{eq:blend_stats}).}
    \label{fig:training_dynamics}
\end{figure*}

\begin{figure*}[t]
    \centering
    \includegraphics[width=0.98\textwidth]{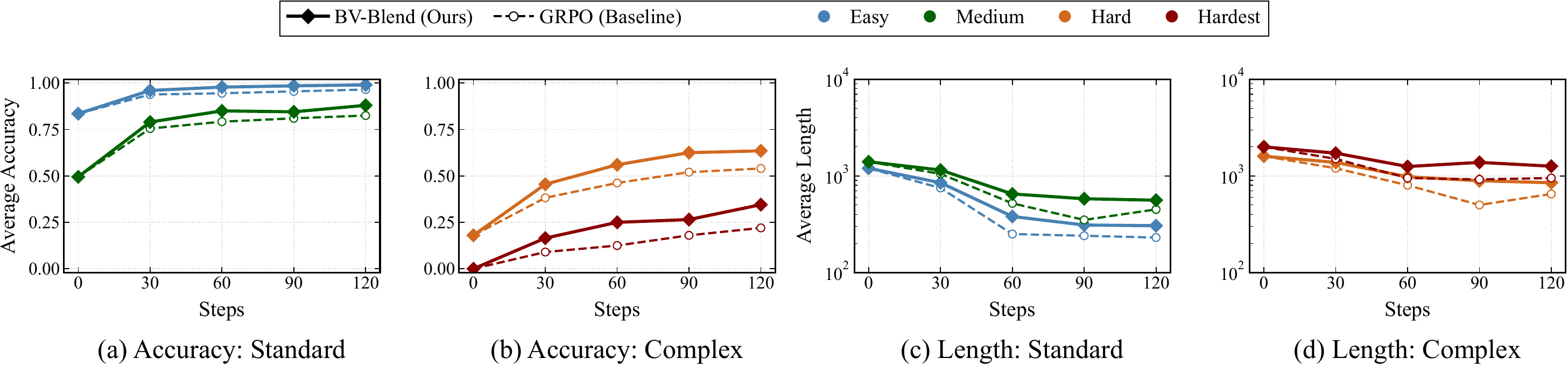}
    \caption{\textbf{Difficulty-stratified \method\ vs.\ GRPO.}
    We partition prompts into four difficulty buckets (Easy/Medium/Hard/Hardest) using a fixed pre-RL difficulty estimate (Appendix~\ref{app:difficulty_stratification}) shared across methods, and track \emph{verifier accuracy} (a,b; fraction of prompts with correct final answers) and \emph{average response length} (c,d) across training checkpoints. The left pair (a,c) reports the \textit{Standard} subset and the right pair (b,d) reports the \textit{Complex} subset. Response length is measured on completion tokens and plotted on a log scale.}
    \label{fig:difficulty_dynamics}
\end{figure*}

\paragraph{Implementation Details.}
We evaluate \method across multiple backbones. Our primary testbed is Qwen2.5-Math-7B to align with prior RLVR work. We additionally report results on Qwen2.5-Math-1.5B, Qwen2.5-7B-Instruct, and Llama-3.1-8B-Instruct to assess robustness across model scales and instruction tuning. Unless otherwise specified, we use a shared training setup across all runs. We optimize with AdamW and a cosine learning-rate schedule with linear warmup, with a peak learning rate of \(1\times10^{-6}\). Each iteration collects a global rollout batch of 128 trajectories (16 prompts \(\times\) 8 rollouts) with sampling temperature \(1.0\), and performs policy updates with a trajectory minibatch size of 64. Rewards are binary: we assign \(+1\) if the extracted final answer is verified as correct by \textit{Math-Verify}, and \(0\) otherwise, with no intermediate or format-based shaping rewards. We use a PPO-style clipped objective shared by \method and the PPO-style baselines, with clipping \(\epsilon=0.2\) and an entropy bonus coefficient of \(0.01\). Unless otherwise stated, we include a KL-to-reference term in the objective but set its coefficient to \(\beta=0\) in the main experiments, effectively disabling it to isolate the effect of different advantage estimators. For \method, we maintain cluster-conditioned historical moments using EMA with update rate \(\gamma=0.9\), and use \(\delta=10^{-8}\) for numerical stability in advantage computation. The remaining \method-specific hyperparameters and clustering settings are fixed across experiments unless otherwise stated. We do not apply any additional global advantage normalization beyond the method-specific estimator (e.g., GRPO or \method). For purely on-policy approaches (including \method and our GRPO replication), all 8 rollouts per prompt are generated on-policy from the current policy. All models are trained for 500 iterations. All experiments are conducted on 8 NVIDIA RTX PRO 6000 GPUs (96\,GB VRAM each).

\subsection{Main Results}
\label{sec:main_results}

Table~\ref{tab:main_results} summarizes results on the Qwen2.5-Math-7B backbone.
Overall, \method achieves the best average performance among the compared methods on both in-domain mathematical reasoning and OOD benchmarks.
Notably, \method is a purely on-policy RLVR approach, yet it remains competitive with hybrid/off-policy systems that additionally leverage external data.

\paragraph{In-Domain Mathematical Reasoning.}
We evaluate in-domain performance on five math benchmarks, where AIME is reported for two years (AIME 2024/2025), yielding six scores in total for the Math Reasoning Avg.
AIME and AMC are reported as avg@32, while the remaining benchmarks use pass@1.
Across these six scores, \method attains the highest average (51.7\%), outperforming the strongest on-policy baseline Oat-Zero (43.8\%) by 7.9 points.
It also exceeds ReLIFT (47.8\%) and LUFFY (50.1\%), and is slightly above the extended-training LUFFY$\dagger$ variant (51.0\%).
At the per-benchmark level, \method achieves the best reported results on AIME 2024 (34.2\%), AMC (66.5\%), MATH-500 (87.9\%), and Olympiad (57.4\%).
Meanwhile, it is marginally below the best baseline on Minerva (40.7\% vs.\ 41.2\%) and on AIME 2025 (23.6\% vs.\ 25.5\%), indicating broad but non-uniform gains across evaluation settings.

\paragraph{OOD Generalization.}
On the OOD suite (ARC-C, GPQA*, and MMLU-Pro), \method again achieves the best average score (64.1\%), improving over the next-best baseline LUFFY$\dagger$ (61.8\%) by 2.3 points.
The largest gains appear on reasoning-centric benchmarks: \method achieves the best results on ARC-C (83.1\%) and GPQA* (52.6\%), exceeding LUFFY$\dagger$ by 3.6 points on GPQA*.
On MMLU-Pro, \method is slightly below the best baseline (56.5\% vs.\ 58.7\%), suggesting that improvements are more pronounced on benchmarks emphasizing multi-step reasoning, while broad knowledge coverage may depend more on pretraining and instruction tuning.

Taken together, Table~\ref{tab:main_results} shows that stabilizing the advantage signal can improve on-policy RLVR training, yielding substantial gains without requiring additional supervision beyond verifiable rewards, while remaining competitive with hybrid/off-policy approaches that leverage extra data.

\begin{table*}[t]
    \centering
    \renewcommand{\arraystretch}{1.15}
    \caption{\textbf{Ablations of \method on Qwen2.5-Math-7B.} We ablate the SEM-based confidence weighting by constructing \(w_k\) from only the EMA effective count \(N_k^{\text{eff}}\) or only the historical scale \(\sigma_{\text{hist}}(k)\). Avg.\ is computed over six scores (AIME 2024, AIME 2025, AMC, MATH-500, Minerva, Olympiad).}
    \label{tab:ablations}
    \resizebox{0.78\textwidth}{!}{
    \setlength{\tabcolsep}{3.0pt}
    \begin{tabular}{lcccccc}
    \toprule
    \textbf{Model} & \textbf{AIME 24/25} & \textbf{AMC} & \textbf{MATH-500} & \textbf{Minerva} & \textbf{Olympiad} & \textbf{Avg.} \\
    \midrule
    GRPO (baseline) & 25.1/15.3 & 62.1 & 84.3 & 39.3 & 46.9 & 45.5 \\
    Naive historical averaging ($w{=}0.5$) & 25.8/16.1 & 58.2 & 79.6 & 37.9 & 44.5 & 43.7 \\
    \midrule
    \method ($w_k$ from $N_k^{\text{eff}}$ only) & 31.5/22.2 & 64.2 & 86.8 & 39.6 & 53.9 & 49.7 \\
    \method ($w_k$ from $\sigma_{\text{hist}}$ only) & 29.1/21.3 & 65.2 & 87.3 & 40.3 & 51.2 & 49.1 \\
    \hdashline
    \rowcolor{green!10}
    \textbf{\method (full SEM-based confidence)} & \textbf{34.2}/\textbf{23.6} & \textbf{66.4} & \textbf{87.9} & \textbf{40.7} & \textbf{57.4} & \textbf{51.7} \\
    \bottomrule
    \end{tabular}}
\end{table*}

\subsection{Analysis of Training Dynamics}
\label{sec:training_dynamics}

To diagnose how \method shapes optimization, Fig.~\ref{fig:training_dynamics} compares BV-Blend with the on-policy baseline GRPO across four training signals.
BV-Blend consistently produces longer responses and reaches a higher, more stable length plateau (Fig.~\ref{fig:resp_len}), which is consistent with sustaining multi-step reasoning traces under the same rollout budget.
While entropy decreases for both methods early in training, BV-Blend maintains a higher residual policy entropy throughout (Fig.~\ref{fig:train_ent}), which is consistent with slower policy concentration and more persistent exploration.
This is accompanied by slightly higher and noticeably smoother training rewards (Fig.~\ref{fig:train_rew}), with reduced volatility relative to GRPO.

Most importantly, the effective-signal ratio (Fig.~\ref{fig:eff_signal}) highlights a difference in the resulting learning signal.
We define the effective-signal ratio as the fraction of prompts whose normalization scale is non-degenerate under the corresponding estimator (GRPO: \(\sigma_{\mathcal{G}}^{(m)}\); BV-Blend: \(s^{(m)}\)).
With binary verifier rewards, prompt groups frequently become near-deterministic (all-correct or all-incorrect), causing \(\sigma_{\mathcal{G}}^{(m)}\) to approach zero and thereby collapsing GRPO advantages toward zero, which reduces the number of informative prompt groups available for learning~\citep{shao2024deepseekmath}.
In contrast, BV-Blend preserves a substantially higher effective-signal ratio by stabilizing the baseline and scale via cluster-conditioned historical moments when prompt-local dispersion is small, thereby maintaining usable learning signals and yielding more stable optimization dynamics overall.

\subsection{Difficulty-stratified dynamics}
\label{sec:difficulty_dynamics}

To better understand where \method's gains arise and whether they are accompanied by undesirable verbosity, we analyze learning dynamics under a fixed difficulty stratification. Specifically, following the protocol described in Appendix~\ref{app:difficulty_stratification}, we assign each prompt to one of four buckets (Easy/Medium/Hard/Hardest) using a pre-RL difficulty estimate, and evaluate both \method\ and the on-policy baseline GRPO at the same checkpoints (Steps $0/30/60/90/120$). Fig.~\ref{fig:difficulty_dynamics} reports bucket-wise average verifier accuracy and average completion length (log scale) on both \textit{Standard} and \textit{Complex} subsets. Overall, \method\ yields consistent accuracy improvements on harder buckets—most notably Hard and Hardest—while keeping response lengths comparable and avoiding systematic length blow-up. We observe mild non-monotonic fluctuations across checkpoints, which is expected in on-policy optimization due to sampling noise; importantly, \method\ shows a more consistent upward trend on difficult prompts, consistent with the hypothesis that historically stabilized advantage estimation improves training stability without evidence of systematic length inflation.

\subsection{Ablation Study}
\label{sec:ablation}

Table~\ref{tab:ablations} ablates the core design of \method on Qwen2.5-Math-7B. We compare GRPO, a naive fixed-weight historical mixture (\(w{=}0.5\)), two partial variants that compute the confidence weight using only one historical-uncertainty ingredient, and the full \method.
All \method variants use the same baseline/scale blending in Eq.~\eqref{eq:blend_stats} and differ only in how \(w_k\) is computed from historical statistics.
Concretely, the full method computes the SEM-style uncertainty
\(\mathrm{SEM}_{\mathrm{hist}}(k)=\sigma_{\mathrm{hist}}(k)/\sqrt{N_k^{\mathrm{eff}}+\delta_N}\)
and maps it to confidence via Eq.~\eqref{eq:weight}.
The \(N_k^{\mathrm{eff}}\)-only variant drops the dependence on \(\sigma_{\mathrm{hist}}(k)\) (i.e., \(\mathrm{SEM}(k)\propto 1/\sqrt{N_k^{\mathrm{eff}}+\delta_N}\)),
while the \(\sigma_{\mathrm{hist}}\)-only variant drops the dependence on \(N_k^{\mathrm{eff}}\) (i.e., \(\mathrm{SEM}(k)\propto \sigma_{\mathrm{hist}}(k)\)); in all cases, \(w_k\) is a monotone function of the corresponding uncertainty proxy.

Two observations emerge. First, simply injecting historical information is insufficient: naive averaging drops from 45.5 (GRPO) to 43.7, indicating that historical moments should be used selectively rather than uniformly. Second, either signal alone already improves over GRPO (49.7 and 49.1), and combining them through the SEM proxy performs best (51.7), improving by 6.2 points over GRPO and by 2.0 points over the best single-signal variant. The gains are particularly noticeable on harder benchmarks, such as Olympiad and AIME, which is consistent with the role of SEM-based confidence in stabilizing advantage estimation when prompt-local reward dispersion is small.

\section{Conclusion}
We identify an instability in critic-free RLVR: when prompt-local reward dispersion is small, group-normalized advantages can collapse and weaken the learning signal.
We propose \method\ (BV-Blend), which stabilizes advantage estimation by blending prompt-local statistics with semantic-cluster-conditioned historical moments through an SEM-based confidence weight, while leaving the PPO-style clipped objective unchanged.
Across experiments, \method\ improves training stability and performance on in-domain, OOD, and cross-backbone evaluations.
Future work will explore more effective mechanisms for sharing historical information across prompts, especially under distribution shift.

\section*{Limitations}
\method\ relies on a fixed prompt embedding and clustering pipeline; performance may depend on the encoder and codebook granularity, and may degrade under substantial distribution shift.
It also requires maintaining cluster-conditioned EMA moments and aggregating per-cluster statistics in distributed training, which introduces additional bookkeeping and hyperparameter tuning.
Finally, we primarily evaluate mathematical reasoning with verifiable outcome rewards and a limited set of backbones; broader domains and verifier/reward designs are needed to more fully assess generality.
In addition, under extremely sparse reward regimes, if both the current prompt group and its relevant historical cluster provide little or no reward variation, BV-Blend may still offer limited learning signal.

\section*{Acknowledgments}
This work is supported by the National Key Research and Development Program of China (No.2023YFF0905400), the National Natural Science Foundation of China (No.U2341229) and the Reform Commission Foundation of Jilin Province (No.2024C003).

\bibliography{references}

\newpage
\appendix

\vspace{2em}
\begin{center}
    \Large{\textbf{Appendix}}
\end{center}
\vspace{2em}

\etocdepthtag.toc{appendix}
\etocsettagdepth{chapter}{none}
\etocsettagdepth{appendix}{subsection}
\tableofcontents

\section{Experimental Setup}
\label{subsec:experimental_setup}

This appendix describes the experimental environment, task formulation, datasets, evaluation protocols, and base model configurations. Our goal is to provide sufficient detail for reproducibility.

\subsection{Problem Formulation as a Markov Decision Process (MDP)}

We model autoregressive generation as a finite-horizon episodic Markov Decision Process (MDP)
\((\mathcal{S}, \mathcal{A}, P, r, H)\), where \(H\) is the maximum number of \emph{generated completion tokens} (including EOS when present).

\begin{itemize}
    \item \textbf{State Space (\(\mathcal{S}\)).}
    A state at step \(t\) is the prompt concatenated with the generated prefix:
    \(s_t = [q; y_{<t}] = [q; y_1; \dots; y_{t-1}]\), with initial state \(s_1 = [q]\).

    \item \textbf{Action Space (\(\mathcal{A}\)).}
    The action space is the model vocabulary. At state \(s_t\), the policy \(\pi_\theta(\cdot \mid s_t)\)
    outputs a distribution over tokens, and we sample an action \(a_t = y_t\).

    \item \textbf{Transition Dynamics (\(P\)).}
    Transitions are deterministic: after taking action \(a_t=y_t\), the next state is
    \(s_{t+1} = [q; y_{\le t}] = [s_t; y_t]\).

    \item \textbf{Reward (\(r\)) and Return.}
    We use sparse, outcome-based rewards. A scalar reward is given only at episode termination and is binary:
    \[
    R(\tau)=
    \begin{cases}
    1 & \text{if } \textit{Math-Verify}(\tau)=\text{True},\\
    0 & \text{otherwise},
    \end{cases}
    \]
    where \(\tau=(y_1,\dots,y_T)\) denotes the generated completion and \textit{Math-Verify} returns \texttt{True} if the extracted final answer is judged equivalent to the gold answer under its comparison rules.
    No intermediate rewards are provided (i.e., \(r_t=0\) for \(t<T\)), and the terminal reward equals \(r_T = R(\tau)\).
    In optimization (Sec.~\ref{sec:method}), this trajectory-level reward induces a single per-trajectory advantage that is applied to completion tokens, with gradients stopped through all reward statistics.

    \item \textbf{Termination and Horizon (\(H\)).}
    An episode terminates when either (i) the model emits an end-of-sequence token (EOS), or (ii) the number of generated completion tokens reaches the maximum length limit \(H\).
    We set \(H=8192\) in all experiments.
\end{itemize}

\subsection{Justification for the Mathematical Reasoning Domain}

We select mathematical reasoning as our primary experimental domain for three reasons:
\begin{enumerate}
    \item \textbf{Objective, verifiable rewards.}
    Mathematical problems admit a well-defined notion of correctness, enabling programmatic outcome verification against ground-truth final answers (e.g., by extracting the final answer, normalizing it, and checking equivalence).
    This reduces reliance on subjective annotations or learned judges that may introduce noise or exploitable biases.

    \item \textbf{A natural stress test for prompt-local normalization under sparse binary rewards.}
    With sparse binary rewards, difficult prompts often yield rollout groups where most samples share the same outcome (e.g., all incorrect), making within-prompt reward dispersion small.
    This setting makes it possible to directly study failure cases of prompt-local normalization methods, including GRPO-style estimators, where the normalization scale can become (near-)degenerate and standardized advantages weaken the effective learning signal (Sec.~\ref{sec:method}).

    \item \textbf{Scalable and rigorous evaluation.}
    Automated verification supports efficient evaluation on large test sets at low cost, improving scalability, consistency, and reproducibility across methods and backbones.
\end{enumerate}

\subsection{Reward Computation and Visualized Example}

\paragraph{Reward Computation.}
We compute sparse, outcome-based rewards using \textit{Math-Verify}.
Given a completed trajectory \(\tau\), \textit{Math-Verify} (i) extracts a candidate final answer from the model output using a priority-based extraction pipeline (e.g., preferring content inside \texttt{\textbackslash boxed\{\(\cdot\)\}} when present), (ii) normalizes the extracted text and parses it into a canonical symbolic representation (e.g., via ANTLR-based parsing and \texttt{latex2sympy2\_extended}), and (iii) checks equivalence against the gold answer using rule-based and SymPy-backed comparisons \citep{mathverify}.
We assign a binary trajectory-level reward
\[
R(\tau)=
\begin{cases}
1 & \text{if } \textit{Math-Verify}(\tau)=\text{True},\\
0 & \text{otherwise},
\end{cases}
\]
where parsing failures or non-equivalence return \(0\).
Some answer representations receive specialized handling (e.g., relations, sets/intervals), following \textit{Math-Verify}'s documented comparison rules \citep{mathverify}. This outcome-based formulation is standard in settings with verifiable rewards \citep{hao2025rethinking} and is broadly aligned with evaluation practices for large language models that emphasize reliable outcome assessment \citep{chang2024survey}. It is also compatible with broader reasoning-oriented frameworks that rely on trajectory-level outcomes \citep{hao2026recreate,chen2026actormindemulatinghumanactor}. Similar considerations also arise in structured reasoning tasks such as table understanding and text-to-SQL, where intermediate reasoning must ultimately support correct final outputs \citep{wu2025table,wu2025ucs,wu2025mr}. More broadly, related retrieval settings also emphasize the importance of aligning intermediate evidence with final decisions \citep{ReTrack,HABIT}. Recent work on controllable and explainable reasoning, adaptive exploration, test-time adaptation, and solution refinement further highlights the value of structured intermediate reasoning for reliable final outcomes \citep{dong2026neureasonerexplainablecontrollableunified,zhou2026look,shen2026preconditioned,jiang2026foeforesterrorsmakes}.

\begin{figure}[h!]
\centering
\caption{An illustrative example showing a prompt, a successful and a failed generation trajectory, and their corresponding rewards under \textit{Math-Verify}.}
\label{fig:appendix_example}
\small
\begin{tabularx}{\columnwidth}{|X|}
\hline
\textbf{Prompt}\\
\hline
Let $f(x) = x^2 - 2x + 3$. Find the vertex of the parabola. \\
\hline
\textbf{Successful Generation Trajectory ($\tau_1$)} \\
\hline
The vertex of a parabola $f(x) = ax^2 + bx + c$ is at $x = -b/(2a)$. Here $a=1, b=-2$. So $x = -(-2)/(2\cdot 1) = 1$. The y-coordinate is $f(1) = 1^2 - 2(1) + 3 = 2$. The vertex is $(1, 2)$.
\texttt{\textbackslash boxed\{}(1, 2)\texttt{\}} \\
\hline
\textbf{Reward for $\tau_1$: $R(\tau_1) = 1$} \\
\hline
\hline
\textbf{Failed Generation Trajectory ($\tau_2$)} \\
\hline
The vertex of a parabola $f(x) = ax^2 + bx + c$ is at $x = b/(2a)$. Here $a=1, b=-2$. So $x = -2/(2\cdot 1) = -1$. The y-coordinate is $f(-1) = (-1)^2 - 2(-1) + 3 = 6$. The vertex is $(-1, 6)$.
\texttt{\textbackslash boxed\{}(-1, 6)\texttt{\}} \\
\hline
\textbf{Reward for $\tau_2$: $R(\tau_2) = 0$} \\
\hline
\end{tabularx}
\end{figure}

\subsection{Base Models and Prompting Format}
To support reproducibility, we list all base models and their sources.

\begin{itemize}
    \item \textbf{Primary Model.}
    Our main experiments and ablations use \texttt{Qwen2.5-Math-7B} \citep{qwen2.5_math}, loaded from the Hugging Face Hub (\texttt{Qwen/Qwen2.5-Math-7B}).

    \item \textbf{Robustness-Test Models.}
    To test robustness across model scales and instruction tuning, we additionally evaluate
    \texttt{Qwen/Qwen2.5-Math-1.5B}, \texttt{Qwen/Qwen2.5-7B-Instruct} \citep{qwen2.5},
    and \texttt{meta-llama/Meta-Llama-3.1-8B-Instruct} \citep{llama3}.
    All checkpoints are obtained from the Hugging Face Hub.

    \item \textbf{Prompt Format.}
    We use each checkpoint's tokenizer-provided chat template to render prompts, via
    \texttt{tokenizer.apply\_chat\_template(\dots, add\_generation\_prompt=True)} for both training and evaluation \citep{hf_chat_templating}.
    A typical rendered prompt (shown here for clarity) has the following structure:
\begin{verbatim}
<|im_start|>system
You are a helpful assistant.<|im_end|>
<|im_start|>user
{problem_description}<|im_end|>
<|im_start|>assistant
\end{verbatim}
    where \texttt{\{problem\_description\}} is the problem text from the dataset.

    \item \textbf{Implementation Note.}
Although our experiments use full-parameter fine-tuning rather than PEFT, they are motivated by the broader goal of robust adaptation for large models. This perspective is related to prior work on reducing adaptation overhead in large models \citep{tian2026grass,dong2025aurora,chang2026balora,chang2025lora} and to continual-adaptation strategies that aim to preserve useful knowledge under iterative updates \citep{chen-zeng-2025-prototype}.
\end{itemize}

\section{Implementation Details and Pseudocode}
\label{subsec:implementation_details}

This section describes the implementation of \method\ (BV-Blend), including the historical-moments buffer, update ordering, and simplified pseudocode for advantage computation and training.

\subsection{Historical Moments Buffer}

BV-Blend maintains \emph{semantic-cluster-conditioned} historical reward moments, rather than per-prompt (UID-level) statistics.
Let the clustering codebook have \(K\) clusters (Sec.~\ref{sec:clustering}).
For each cluster \(k\in\{1,\dots,K\}\), we store three EMA statistics:
(i) the EMA mean \(m_{1}(k)\),
(ii) the EMA second raw moment \(m_{2}(k)\),
and (iii) the EMA effective mass \(N^{\mathrm{eff}}_k\).
In practice, we implement these as dense length-\(K\) buffers
\texttt{m1[K]}, \texttt{m2[K]}, \texttt{n\_eff[K]},
together with a boolean \texttt{seen[K]} flag, which simplifies distributed aggregation and implementation.

\paragraph{Cold-start initialization and update ordering.}
For each batch, BV-Blend computes the confidence weight \(w_k\) from the \emph{pre-update} EMA state (i.e., before incorporating the current batch).
If a cluster \(k\) has not been observed previously (\texttt{seen[k]=False}), we treat it as cold start: we set \(w_k=0\) for this batch (pure prompt-local normalization), and initialize its EMA state \emph{after} the policy update using batch statistics:
\begin{equation}
\begin{aligned}
N^{\mathrm{eff}}_k &\leftarrow N_0,\\
m_{1}(k) &\leftarrow \mu_{\mathcal{B}}(k),\\
m_{2}(k) &\leftarrow \mu_{\mathcal{B}}(k)^2 + V_{\mathrm{prior}}.
\end{aligned}
\end{equation}
For clusters with an existing EMA state, we update moments \emph{post-update} using aggregated sufficient statistics from the current batch (Eq.~\eqref{eq:ema_moments}).

\paragraph{Distributed aggregation.}
Given batch \(\mathcal{B}\), we compute per-cluster sufficient statistics
\(S_{1,\mathcal{B}}(k)=\sum_{(m,i)\in \mathcal{I}_{\mathcal{B}}(k)} R_i^{(m)}\),
\(S_{2,\mathcal{B}}(k)=\sum_{(m,i)\in \mathcal{I}_{\mathcal{B}}(k)} (R_i^{(m)})^2\),
and \(N_{\mathcal{B}}(k)=|\mathcal{I}_{\mathcal{B}}(k)|\).
In distributed training, we all-reduce (sum) these dense buffers across workers before applying EMA updates.

\subsection{BV-Blend Training Loop and Advantage Computation}

Algorithm~\ref{alg:bvblend_training_loop} shows the overall loop.
Algorithm~\ref{alg:bvblend_advantage} details BV-Blend advantage computation, which blends \emph{baseline and variance statistics} and then standardizes once (Eq.~\eqref{eq:blend_stats}--\eqref{eq:bv_adv}).
As in the main text, we stop gradients through all reward statistics, including \(\mu_{\mathcal{G}}^{(m)}\), \(\sigma_{\mathcal{G}}^{(m)}\), and EMA moments.

\begin{algorithm}[t]
\caption{Compute BV-Blend Advantages \(A^{\mathrm{BV}}\) (Simplified)}
\label{alg:bvblend_advantage}
\begin{algorithmic}[1]
\REQUIRE Prompts \(\{q^{(m)}\}\), rewards \(\{R_i^{(m)}\}\), EMA buffers \((m_1,m_2,N^{\mathrm{eff}})\) and flags \texttt{seen}, encoder \(E(\cdot)\), codebook \(\{c_j\}\).
\ENSURE Advantages \(\{A_i^{\mathrm{BV},(m)}\}\).

\FOR{each prompt \(m\)}
    \STATE Compute prompt-local mean/std \(\mu_{\mathcal{G}}^{(m)}, \sigma_{\mathcal{G}}^{(m)}\) from \(\{R_i^{(m)}\}_{i=1}^G\).
    \STATE Assign cluster \(k^{(m)} \leftarrow \arg\min_j \|E(q^{(m)})-c_j\|_2^2\).
    \IF{\texttt{seen[$k^{(m)}$]=False} \hfill (unseen cluster, pre-update)}
        \STATE Set \(w^{(m)} \leftarrow 0\). \hfill (cold start; pure prompt-local)
        \STATE Set \(b^{(m)} \leftarrow \mu_{\mathcal{G}}^{(m)}\), \(s^{(m)} \leftarrow \sigma_{\mathcal{G}}^{(m)}\).
        \STATE Mark \(k^{(m)}\) as ``needs init'' for post-update initialization.
    \ELSE
        \STATE \(\mu_{\mathrm{hist}} \leftarrow m_1(k^{(m)})\);
        \(v_{\mathrm{hist}} \leftarrow \max(m_2(k^{(m)})-m_1(k^{(m)})^2,0)\).
        \STATE \(\mathrm{SEM}_{\mathrm{hist}} \leftarrow \sqrt{v_{\mathrm{hist}}}\big/\sqrt{N^{\mathrm{eff}}_{k^{(m)}}+\delta_N}\).
        \STATE \(w^{(m)} \leftarrow \exp(-\mathrm{SEM}_{\mathrm{hist}}/T)\).
        \STATE Set
        \[
        \begin{aligned}
        b^{(m)} &\leftarrow w^{(m)}\mu_{\mathrm{hist}}
        + (1-w^{(m)})\mu_{\mathcal{G}}^{(m)}, \\
        s^{(m)} &\leftarrow \sqrt{
        w^{(m)}v_{\mathrm{hist}}
        + (1-w^{(m)})(\sigma_{\mathcal{G}}^{(m)})^2
        }.
        \end{aligned}
        \]
    \ENDIF
    \FOR{each trajectory \(i=1,\dots,G\)}
        \STATE \(A_i^{\mathrm{BV},(m)} \leftarrow (R_i^{(m)}-b^{(m)})/(s^{(m)}+\delta)\).
    \ENDFOR
\ENDFOR
\RETURN \(\{A_i^{\mathrm{BV},(m)}\}\).
\end{algorithmic}
\end{algorithm}

\begin{algorithm}[t]
\caption{\method\ (BV-Blend) Training Loop (Simplified)}
\label{alg:bvblend_training_loop}
\begin{algorithmic}[1]
\STATE \textbf{Initialize} policy parameters \(\theta\), reference policy \(\pi_{\mathrm{ref}}\), and fixed clustering codebook \(\{c_j\}_{j=1}^K\).
\STATE \textbf{Initialize} EMA buffers \(\texttt{m1[K]}, \texttt{m2[K]}, \texttt{n\_eff[K]}\) and flags \(\texttt{seen[K]}\leftarrow\texttt{False}\).
\FOR{iteration \(t=1,2,\dots\)}
    \STATE Set \(\theta_{\mathrm{old}} \leftarrow \theta\). \hfill (behavior snapshot)
    \STATE Sample \(M\) prompts \(\{q^{(m)}\}_{m=1}^M\); for each prompt, sample \(G\) trajectories \(\{\tau_i^{(m)}\}_{i=1}^G\) from \(\pi_{\theta_{\mathrm{old}}}\).
    \STATE Compute verifier rewards \(R_i^{(m)}\) for all trajectories.
    \STATE Compute BV-Blend advantages \(A_i^{\mathrm{BV},(m)}\) using Algorithm~\ref{alg:bvblend_advantage} (pre-update EMA).
    \STATE Update \(\theta\) by optimizing the PPO-style clipped objective using \(A_i^{\mathrm{BV},(m)}\) (and optional KL/entropy terms).
    \STATE \textbf{Post-update:} aggregate per-cluster sufficient statistics \(\{S_{1,\mathcal{B}}(k), S_{2,\mathcal{B}}(k), N_{\mathcal{B}}(k)\}_{k=1}^K\) across workers.
    \STATE \textbf{Post-update:} update EMA buffers for clusters with \(N_{\mathcal{B}}(k)>0\) using Eq.~\eqref{eq:ema_moments}; for previously unseen clusters, initialize using \((N_0,V_{\mathrm{prior}})\) and set \(\texttt{seen[k]}\leftarrow\texttt{True}\).
\ENDFOR
\end{algorithmic}
\end{algorithm}

\subsection{Model and Optimization}

\paragraph{Model Architecture.}
We perform full-parameter fine-tuning on all base models.
We do not use parameter-efficient fine-tuning (PEFT) methods such as LoRA~\citep{hu2022lora}, and we do not add adapters or custom heads; the transformer architecture remains unchanged.

\paragraph{Optimization Objective.}
During each policy update, we optimize a PPO-style clipped surrogate objective~\citep{schulman2017ppo} using the BV-Blend advantages \(A_i^{\mathrm{BV},(m)}\) (Sec.~\ref{sec:method}).
Concretely, we maximize
\begin{align}
\mathcal{J}(\theta)
&=
\mathbb{E}_{m}\!\Bigg[
\frac{1}{G}\sum_{i=1}^{G}
\mathbb{E}_{t\sim\tau^{(m)}_i}\!\Big[
\min\Big(
r^{(m)}_{i,t}(\theta)\,A^{\mathrm{BV},(m)}_i,\nonumber\\
&\qquad\qquad\qquad\qquad\qquad
\tilde r^{(m)}_{i,t}(\theta)\,A^{\mathrm{BV},(m)}_i
\Big)
\Big]
\Bigg] \nonumber\\
&\quad
-\beta\;\mathbb{E}_{m,i}\!\Big[
\mathbb{E}_{t\sim\tau^{(m)}_i}\!\big[
\mathrm{KL}\!\Bigl(\pi_\theta(\cdot\mid s_t)\,\Big\|\,\nonumber\\
&\qquad\qquad\qquad\qquad\qquad
\pi_{\mathrm{ref}}(\cdot\mid s_t)\Bigr)
\big]
\Big] \nonumber\\
&\quad
+\lambda_{\mathrm{ent}}\;\mathbb{E}_{m,i}\!\Big[
\mathbb{E}_{t\sim\tau^{(m)}_i}\!\big[
\mathcal{H}\!\left(\pi_\theta(\cdot\mid s_t)\right)
\big]
\Big],
\label{eq:appendix_objective}
\end{align}
where \(r^{(m)}_{i,t}(\theta)\) and \(\tilde r^{(m)}_{i,t}(\theta)\) are defined in Eq.~\eqref{eq:ratio}, and \(\mathbb{E}_{t\sim\tau^{(m)}_i}[\cdot]\) denotes the completion-token mean with the same masking convention as Eq.~\eqref{eq:masked_mean}.
We implement training by minimizing the loss \(\mathcal{L}(\theta)=-\mathcal{J}(\theta)\).
Unless otherwise specified, we set \(\lambda_{\mathrm{ent}}=0.01\) and use a fixed reference policy \(\pi_{\mathrm{ref}}\); in our main comparisons, we set \(\beta=0\).

\subsection{Prompting Strategy}

We format inputs using each model's tokenizer-provided chat template via
\texttt{tokenizer.apply\_chat\_template(\dots, add\_generation\_prompt=True)} for both training and evaluation, which improves reproducibility and avoids formatting mismatches across backbones~\citep{hf_chat_templating}.
Unless otherwise stated, we use a generic assistant system message and provide the problem statement as the user message.
We do not require additional task-specific instructions for Qwen-based math models; they typically produce multi-step solutions by default.
For Llama-3.1-Instruct in the robustness experiments, we optionally prepend a short reasoning cue (e.g., ``Let's think step by step.'') to encourage step-by-step solutions~\citep{wei2022chain}.

\begin{tcolorbox}[
    title=Rendered prompt for Qwen-based chat models (schematic),
    colback=blue!5!white, colframe=blue!75!black, fonttitle=\bfseries,
    arc=2mm, boxrule=1pt,
]
\begin{verbatim}
<|im_start|>system
You are a helpful assistant.<|im_end|>
<|im_start|>user
{problem_description}<|im_end|>
<|im_start|>assistant
\end{verbatim}
\end{tcolorbox}

\begin{tcolorbox}[
    title=Rendered prompt for Llama-3.1-Instruct (schematic),
    colback=red!5!white, colframe=red!75!black, fonttitle=\bfseries,
    arc=2mm, boxrule=1pt,
]
\textbf{User:} \{problem\_description\} \\
\textbf{Assistant:} Let's think step by step.
\end{tcolorbox}

\paragraph{Notation.}
For convenience, Table~\ref{tab:symbols} summarizes the key symbols used throughout the paper, consistent with Sec.~\ref{sec:method}.
We group notation into (i) trajectories and PPO-style optimization, (ii) prompt-local (GRPO-style) reward statistics, and (iii) BV-Blend’s semantic clustering, historical EMA moments, and SEM-based confidence weighting.
Unless otherwise stated, prompts are indexed by \(m\), trajectories by \(i\), completion-token positions by \(t\), clusters by \(k\), and codebook entries by \(j\).
All reward/statistics terms used to form advantages (e.g., \(\mu_{\mathcal{G}}^{(m)}\), \(\sigma_{\mathcal{G}}^{(m)}\), and statistics derived from historical EMA moments) are treated as stop-gradient constants.

\begin{table*}[t]
\centering
\renewcommand{\arraystretch}{1.22}
\caption{Summary of key symbols (aligned with Sec.~\ref{sec:method}).}
\label{tab:symbols}
\begin{tabular}{@{}ll@{}}
\toprule
\textbf{Symbol} & \textbf{Definition} \\ \midrule

\multicolumn{2}{c}{\textit{Trajectories and PPO-style optimization}} \\
\(\pi_{\theta}\), \(\pi_{\theta_{\mathrm{old}}}\) & Current policy and behavior (rollout) policy snapshot. \\
\(\pi_{\mathrm{ref}}\) & Fixed reference policy used for optional KL regularization. \\
\(M\) & Number of prompts in a training batch. \\
\(q^{(m)}\) & The \(m\)-th prompt in a training batch. \\
\(G\) & Number of rollouts (trajectories) per prompt. \\
\(\mathcal{G}(q^{(m)})\) & Rollout group for prompt \(q^{(m)}\): \(\{\tau_i^{(m)}\}_{i=1}^{G}\). \\
\(\tau_i^{(m)}\) & The \(i\)-th trajectory sampled for prompt \(q^{(m)}\). \\
\(T_i^{(m)}\) & Number of generated tokens in trajectory \(\tau_i^{(m)}\). \\
\(R_i^{(m)}\) & Trajectory-level verifier reward for \(\tau_i^{(m)}\). \\
\(r^{(m)}_{i,t}(\theta)\), \(\tilde r^{(m)}_{i,t}(\theta)\) & Importance ratio and its clipped version (Eq.~\eqref{eq:ratio}). \\
\(\epsilon\) & PPO clipping parameter. \\
\(m^{(m)}_{i,t}\) & Completion-token mask used in token means (Eq.~\eqref{eq:masked_mean}). \\
\(\mathbb{E}_{t\sim\tau}[\,\cdot\,]\) & Generic notation for the masked mean over completion tokens (Eq.~\eqref{eq:masked_mean}). \\

\midrule
\multicolumn{2}{c}{\textit{Prompt-local (GRPO-style) statistics}} \\
\(\mu_{\mathcal{G}}^{(m)}\), \(\sigma_{\mathcal{G}}^{(m)}\) & Prompt-local mean/std of \(\{R_i^{(m)}\}_{i=1}^G\). \\
\(A_i^{\mathrm{GRPO},(m)}\) & GRPO-style standardized advantage (Eq.~\eqref{eq:grpo_adv}). \\
\(\delta\) & Small constant for numerical stability in advantage computation. \\

\midrule
\multicolumn{2}{c}{\textit{BV-Blend: clustering and historical EMA moments}} \\
\(E(\cdot)\) & Frozen prompt encoder. \\
\(K\), \(\{c_j\}_{j=1}^K\) & Number of clusters and fixed \(K\)-means codebook. \\
\(k^{(m)}\) & Cluster assignment for prompt \(q^{(m)}\) (Eq.~\eqref{eq:cluster}). \\
\(m_{1}(k)\), \(m_{2}(k)\) & EMA mean and EMA second raw moment for cluster \(k\). \\
\(N_k^{\mathrm{eff}}\) & EMA effective mass for cluster \(k\). \\
\(\mu_{\mathrm{hist}}(k)\), \(v_{\mathrm{hist}}(k)\), \(\sigma_{\mathrm{hist}}(k)\) & Historical mean/variance/std derived from EMA moments (Eq.~\eqref{eq:hist_var}). \\
\(\gamma\) & EMA update rate (Eq.~\eqref{eq:ema_moments}). \\
\(\mathcal{B}\) & Current rollout batch. \\
\(S_{1,\mathcal{B}}(k)\), \(S_{2,\mathcal{B}}(k)\), \(N_{\mathcal{B}}(k)\) & Per-cluster sufficient statistics in batch \(\mathcal{B}\). \\
\(N_0\), \(V_{\mathrm{prior}}\) & Cold-start initialization hyperparameters for new clusters. \\

\midrule
\multicolumn{2}{c}{\textit{BV-Blend: confidence and advantage}} \\
\(\mathrm{SEM}_{\mathrm{hist}}(k)\) & SEM-style uncertainty proxy (Eq.~\eqref{eq:sem}). \\
\(\delta_N\) & Small constant in SEM computation (Eq.~\eqref{eq:sem}). \\
\(T\) & Temperature controlling sensitivity in \(w_k\) (Eq.~\eqref{eq:weight}). \\
\(w_k\) & Historical confidence weight computed from pre-update EMA statistics. \\
\(w^{(m)}\) & Confidence weight for prompt \(m\): \(w^{(m)} = w_{k^{(m)}}\). \\
\(b^{(m)}\), \(s^{(m)}\) & Blended baseline and scale for prompt \(m\) (Eq.~\eqref{eq:blend_stats}). \\
\(A_i^{\mathrm{BV},(m)}\) & BV-Blend advantage (Eq.~\eqref{eq:bv_adv}). \\

\bottomrule
\end{tabular}
\end{table*}

\section{Extended Experimental Results}
\label{subsec:extended_results}

This section presents supplementary quantitative results and ablations that complement the main paper.
Throughout, we compare against the on-policy GRPO baseline under the same PPO-style clipped objective and vary only the advantage estimator.

\subsection{Further Ablation Studies}

\paragraph{Confidence-to-weight mapping.}
We ablate the functional form that maps historical uncertainty to the confidence weight \(w_k\).
All variants compute the same SEM-style uncertainty \(\mathrm{SEM}_{\mathrm{hist}}(k)\) from \emph{pre-update} EMA statistics (Eq.~\eqref{eq:sem}) and use the same baseline/variance blending (Eq.~\eqref{eq:blend_stats}); they differ only in the mapping \(w_k=g(u_k)\), where \(u_k \triangleq \mathrm{SEM}_{\mathrm{hist}}(k)/T\).
Table~\ref{tab:ablation_w} reports results on AIME 2024 (avg@32).
Among the tested monotone mappings, the exponential form \(w_k=\exp(-u_k)\) performs best, which is consistent with the use of a smooth confidence schedule rather than a hard threshold.

\begin{table*}[h!]
\centering
\caption{Ablation on the confidence mapping \(w_k=g(u_k)\) on AIME 2024 (avg@32), where \(u_k=\mathrm{SEM}_{\mathrm{hist}}(k)/T\).}
\label{tab:ablation_w}
\begin{tabular}{l c}
\toprule
\textbf{Confidence Mapping \(g(u)\)} & \textbf{AIME 2024 (\%)} \\
\midrule
\(1 / (1 + u)\) & 32.5 \\
Linear decay: \(\max(0, 1 - u)\) & 30.1 \\
\textbf{Exponential decay: \(\exp(-u)\) (ours)} & \textbf{34.2} \\
\bottomrule
\end{tabular}
\end{table*}

\begin{figure*}[t]
    \centering
    \includegraphics[width=0.93\textwidth]{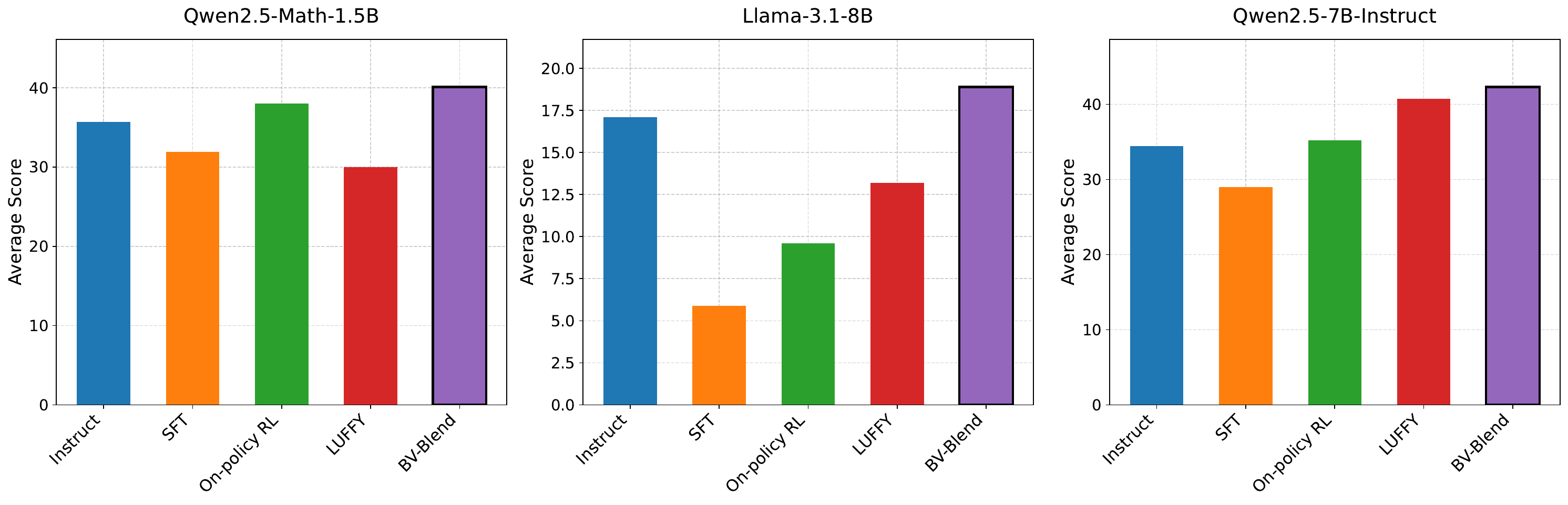}
    \caption{\textbf{Cross-backbone robustness.} Performance of \method\ relative to baselines across diverse model backbones under the same RLVR setup.}
    \label{fig:model_robustness}
\end{figure*}

\subsection{Robustness Across Diverse Models}
\label{sec:robustness_models}

To assess robustness beyond our primary Qwen2.5-Math-7B setting, we evaluate \method on three additional backbones: Qwen2.5-Math-1.5B, Qwen2.5-7B-Instruct, and LLaMA-3.1-8B.
As shown in Fig.~\ref{fig:model_robustness}, \method improves over the best baseline included in our comparison on all three settings: +2.1 points on Qwen2.5-Math-1.5B and +1.6 points on Qwen2.5-7B-Instruct.
The most challenging case is LLaMA-3.1-8B, where standard prompt-local estimators exhibit pronounced instability under our RLVR setup and can substantially degrade final performance; in this regime, \method reaches 19.9, improving by 2.8 points over the best-performing baseline.
While absolute performance on LLaMA-3.1-8B remains modest under this setup, the relative trend is still informative: \method reduces degradation and remains more stable than prompt-local normalization when within-group reward dispersion is small.
Overall, these results suggest that \method is a practical replacement for prompt-local advantage normalization in on-policy RLVR across heterogeneous backbones and training conditions.

\subsection{Proof that GRPO Advantages Vanish Under Uniform Prompt-Group Rewards}
\label{subsubsec:proof_info_collapse}

We show that the GRPO-style prompt-local standardized advantage is exactly zero when all rollouts within a prompt group receive the same reward.

\begin{theorem}[Prompt-local collapse in GRPO]
\label{thm:grpo_collapse}
Fix a prompt \(q^{(m)}\) with rollout set \(\mathcal{G}(q^{(m)})=\{\tau_i^{(m)}\}_{i=1}^{G}\).
If \(R_i^{(m)}=c\) for all \(i\in\{1,\dots,G\}\) and some constant \(c\), then
\(A_i^{\mathrm{GRPO},(m)}=0\) for all \(i\).
\end{theorem}

\begin{proof}
By Eq.~\eqref{eq:grpo_adv},
\[
A_i^{\mathrm{GRPO},(m)}=\frac{R_i^{(m)}-\mu_{\mathcal{G}}^{(m)}}{\sigma_{\mathcal{G}}^{(m)}+\delta},
\]
where \(\mu_{\mathcal{G}}^{(m)}\) and \(\sigma_{\mathcal{G}}^{(m)}\) are the mean and standard deviation of \(\{R_i^{(m)}\}_{i=1}^{G}\), and \(\delta>0\).
If \(R_i^{(m)}=c\) for all \(i\), then \(\mu_{\mathcal{G}}^{(m)}=c\) and \(\sigma_{\mathcal{G}}^{(m)}=0\).
Substituting yields
\[
A_i^{\mathrm{GRPO},(m)}=\frac{c-c}{0+\delta}=0,
\]
for all \(i\).
\end{proof}

\subsection{Difficulty Stratification and Subset Definition}
\label{app:difficulty_stratification}

For the analysis in Fig.~3, we assign each prompt to one of four fixed difficulty buckets (Easy, Medium, Hard, Hardest) using a pre-RL difficulty estimate computed before policy optimization. The same bucket assignment is shared across all compared methods and kept fixed throughout training and evaluation. We further partition the evaluation prompts into \textit{Standard} and \textit{Complex} subsets using the same precomputed protocol, and report checkpoint-wise verifier accuracy and average completion length separately for the two subsets. This analysis is intended only to provide a consistent difficulty-stratified comparison between \method and GRPO; it does not affect training, sampling, or model selection.

\section{Bias--Variance Analysis of the BV-Blend Advantage Estimator}
\label{subsubsec:bias_variance}

This section analyzes the statistical behavior of BV-Blend's advantage construction.
Rather than mixing two advantage estimators directly, BV-Blend blends a \emph{baseline} and \emph{variance statistics} (which determine the scale) and then performs a single standardization:
\begin{equation}
A_i^{\mathrm{BV},(m)} \;=\; \frac{R_i^{(m)} - b^{(m)}}{s^{(m)}+\delta},
\label{eq:bv_adv_appendix}
\end{equation}
where \((b^{(m)}, s^{(m)})\) are defined by Eq.~\eqref{eq:blend_stats}.
Throughout, we treat the confidence weight \(w_k\) as fixed within the current batch because it is computed from \emph{pre-update} EMA moments (Sec.~\ref{sec:uncertainty}) and gradients are stopped through all reward statistics.

\paragraph{Conditional mean and variance (given \(b^{(m)},s^{(m)}\)).}
Conditioned on the (stop-gradient) statistics \((b^{(m)}, s^{(m)})\), BV-Blend is an affine transform of the trajectory reward:
\begin{align}
\mathbb{E}\!\left[A_i^{\mathrm{BV},(m)} \mid b^{(m)}, s^{(m)}\right]
&= \frac{\mathbb{E}[R_i^{(m)}]- b^{(m)}}{s^{(m)}+\delta}, \label{eq:cond_mean}\\
\mathrm{Var}\!\left(A_i^{\mathrm{BV},(m)} \mid b^{(m)}, s^{(m)}\right)
&= \frac{\mathrm{Var}(R_i^{(m)})}{(s^{(m)}+\delta)^2}. \label{eq:cond_var}
\end{align}
Thus, \(s^{(m)}\) directly controls the magnitude of the learning signal entering the PPO-style surrogate objective.

\paragraph{Why BV-Blend can mitigate prompt-local collapse.}
Under GRPO-style prompt-local normalization, \(s^{(m)}=\sigma_{\mathcal{G}}^{(m)}\).
If rewards in a prompt group are constant (all-correct or all-incorrect), then \(\sigma_{\mathcal{G}}^{(m)}=0\) and the standardized advantages are exactly zero (Theorem~\ref{thm:grpo_collapse}), eliminating the learning signal for that group.

In BV-Blend, the blended scale satisfies
\[
(s^{(m)})^2
= w^{(m)} v_{\mathrm{hist}}(k^{(m)}) + (1-w^{(m)})(\sigma_{\mathcal{G}}^{(m)})^2.
\]
Therefore, when \(\sigma_{\mathcal{G}}^{(m)}\approx 0\), the scale remains non-degenerate whenever \(w^{(m)}>0\) and \(v_{\mathrm{hist}}(k^{(m)})>0\).
If a prompt group has constant reward \(R_i^{(m)}=c\) for all \(i\), then
\[
\begin{aligned}
b^{(m)} &= (1-w^{(m)})c + w^{(m)}\mu_{\mathrm{hist}}(k^{(m)}), \\
s^{(m)} &= \sqrt{w^{(m)}v_{\mathrm{hist}}(k^{(m)})},
\end{aligned}
\]
which yields
\begin{equation}
A_i^{\mathrm{BV},(m)}
=
\frac{w^{(m)}\bigl(c-\mu_{\mathrm{hist}}(k^{(m)})\bigr)}{\sqrt{w^{(m)}v_{\mathrm{hist}}(k^{(m)})}+\delta}.
\label{eq:nondegenerate_constant_reward}
\end{equation}
Thus, BV-Blend can retain a non-zero advantage even when prompt-local dispersion vanishes, provided \(w^{(m)}>0\) and \(v_{\mathrm{hist}}(k^{(m)})>0\); conversely, if \(w^{(m)}=0\) (e.g., an unseen cluster) or \(v_{\mathrm{hist}}(k^{(m)})=0\), the advantage may still collapse for that batch.

\paragraph{A bias--variance viewpoint via moment shrinkage.}
BV-Blend can be viewed as constructing \emph{shrunk} estimators of the reward mean and variance (within a semantic cluster), and then standardizing once.
Let \(\mu_\star(k)\) and \(v_\star(k)\) denote the (hypothetical) population moments for cluster \(k\) under the current policy, and let \(\mu_{\mathcal{G}}^{(m)}\) and \((\sigma_{\mathcal{G}}^{(m)})^2\) be the prompt-local sample moments computed from \(G\) rollouts.
BV-Blend uses
\begin{align}
b^{(m)} &= w^{(m)} \mu_{\mathrm{hist}}(k^{(m)}) + (1-w^{(m)})\mu_{\mathcal{G}}^{(m)}, \label{eq:b_shrink}\\
(s^{(m)})^2 &= w^{(m)} v_{\mathrm{hist}}(k^{(m)}) + (1-w^{(m)})(\sigma_{\mathcal{G}}^{(m)})^2, \label{eq:s2_shrink}
\end{align}
i.e., convex combinations of a low-variance historical estimate (aggregating many past samples) and a high-variance prompt-local estimate (based on \(G\) rollouts).
For intuition, under the approximation that (i) \(w^{(m)}\) is fixed within the batch and (ii) the prompt-local moments are approximately unbiased for \((\mu_\star,v_\star)\) up to standard finite-sample effects, the bias of \(b^{(m)}\) inherits the historical bias scaled by \(w^{(m)}\):
\begin{equation}
\mathrm{Bias}\!\left(b^{(m)}\right) \approx w^{(m)}\,\mathrm{Bias}\!\left(\mu_{\mathrm{hist}}(k^{(m)})\right).
\label{eq:bias_mean}
\end{equation}
Similarly, the variance of \(b^{(m)}\) is reduced relative to the prompt-local estimator when the historical estimate is substantially more certain:
\begin{equation}
\begin{aligned}
\mathrm{Var}\!\left(b^{(m)}\right)
&\approx (1-w^{(m)})^2 \mathrm{Var}\!\left(\mu_{\mathcal{G}}^{(m)}\right) \\
&\quad + (w^{(m)})^2 \mathrm{Var}\!\left(\mu_{\mathrm{hist}}(k^{(m)})\right),
\end{aligned}
\label{eq:var_mean}
\end{equation}
where we typically expect the cross-covariance to be small because \(\mu_{\mathrm{hist}}\) is computed from past batches.
These relations are intended as explanatory approximations; they ignore estimation error in \(w^{(m)}\) and \(s^{(m)}\), as well as the additional nonlinearity introduced by the square root and the final normalization.

\paragraph{Why SEM-based weighting is a useful proxy.}
BV-Blend sets \(w_k\) as a monotone function of the SEM-style uncertainty proxy
\(\mathrm{SEM}_{\mathrm{hist}}(k)=\sigma_{\mathrm{hist}}(k)/\sqrt{N_k^{\mathrm{eff}}+\delta_N}\) (Eq.~\eqref{eq:sem}),
matching the classical scaling of the standard error of the sample mean with the standard deviation and sample size.
This makes \(w_k\) large when historical estimates are both low-variance and supported by a large effective count, and small otherwise, thereby adapting the shrinkage strength to estimated uncertainty.

\subsection{Discussion of Key Assumptions}
\label{subsubsec:bias_variance_assumptions}

The analysis above is intended as an explanatory approximation and relies on standard assumptions:
\begin{itemize}
    \item \textbf{Stop-gradient statistics.} We treat \((b^{(m)},s^{(m)},w_k)\) as fixed w.r.t.\ policy gradients, consistent with our implementation (Sec.~\ref{sec:method}).
    \item \textbf{Approximate independence across batches.} Historical EMA moments are computed from past batches, so their covariance with current prompt-local moments is often small, though not exactly zero.
    \item \textbf{Limited non-stationarity.} Historical moments are most informative when the reward distribution within a semantic cluster does not shift arbitrarily fast; the SEM-based weight is designed to reduce reliance on history when uncertainty is high.
    \item \textbf{No strict unbiasedness claim.} Because BV-Blend uses normalization and heuristic shrinkage (and PPO further clips the objective), we do not claim that the resulting standardized quantity is a strictly unbiased estimator of the true MDP advantage.
\end{itemize}

\section{Reproducibility Details}
\label{subsec:computational_resources}

All experiments were conducted on a single multi-GPU node. We report the main hardware and software stack to facilitate reproduction.

\paragraph{Hardware.}
Experiments were conducted on one node equipped with 8 NVIDIA RTX PRO 6000 GPUs (96\,GB memory per GPU).
Training uses PyTorch Fully Sharded Data Parallel (FSDP) across all 8 GPUs.

\paragraph{Software.}
Our implementation is based on PyTorch~2.4.0 and the Hugging Face Transformers/Accelerate stack. For generation, we use vLLM~(v0.6.3) with CUDA~12.1. We additionally use FlashAttention~(v2.7.3), TensorDict~(v0.5.0), and the \texttt{verl} library for RL orchestration. All experiments are conducted in a containerized CUDA~12.1 environment.

\end{document}